\documentclass{article}

 \usepackage[preprint]{neurips_2026}

\usepackage[utf8]{inputenc}
\usepackage[T1]{fontenc}
\usepackage{amsmath,amssymb,amsfonts}
\usepackage{booktabs}
\usepackage{graphicx}
\usepackage{hyperref}
\usepackage{amsmath}
\usepackage{amssymb}
\usepackage{amsthm}

\newtheorem{theorem}{Theorem}
\newtheorem{proposition}{Proposition}
\usepackage{url}
\usepackage{multirow}
\usepackage{xcolor}
\usepackage{subcaption}
\usepackage{enumitem}
\usepackage{bm}
\usepackage{nicefrac}
\usepackage{microtype}
\usepackage[capitalize,noabbrev]{cleveref}

\newcommand{\method}{DeCoDrift}

\title{DeCoDrift: Stabilizing Decoder Coupling in Closed-Loop Foundation Segmentation}

\author{
  H.\,M.~Shadman Tabib\textsuperscript{$\dagger$} \quad
  Md.~Shamsuzzoha Bayzid\textsuperscript{$\dagger$} \quad
  M~Sohel Rahman\textsuperscript{$\dagger$} \\[4pt]
  \textsuperscript{$\dagger$}Department of Computer Science and Engineering \\
  Bangladesh University of Engineering and Technology, Dhaka, Bangladesh \\
}

\begin{document}
\maketitle

\begin{abstract}
Foundation segmentation models such as Segment Anything Model (SAM) are now routinely used in iterative pipelines, where each predicted mask is fed back as the next prompt. This practice turns segmentation into a closed-loop dynamical process, yet the decoder-level behavior of these systems remains largely unexamined. We show that this feedback loop can induce a previously overlooked failure mode, decoder coupling drift, in which the mask decoder’s cross-attention progressively loses alignment with the target object, causing errors to accumulate across iterations. We study this phenomenon by instrumenting SAM’s mask decoder and deriving ground-truth-free measures of prompt-image coupling, attention stability, and temporal consistency. On volumetric electron microscopy data, these decoder-internal signals reveal that standard iterative prompting systematically degrades attention alignment and temporal coherence relative to oracle-anchored feedback. We then formalize iterative prompting as a discrete-time dynamical system and show how proximal anchoring reduces error amplification in the feedback loop. Building on this analysis, we introduce DeCoDrift, a training-free inference-time stabilization framework that constrains prompt updates and preserves decoder coupling across iterations.Across extensive experiments, DeCoDrift consistently improves attention stability, temporal coherence, and segmentation quality over standard iterative prompting, without retraining or ground-truth supervision. More broadly, our results show that decoder-internal dynamics are not merely diagnostic: they provide actionable signals for stabilizing foundation segmentation models in closed-loop use.
\end{abstract}

\section{Introduction}
\label{sec:intro}

Promptable foundation models for image segmentation~\cite{kirillov2023sam,ravi2024sam2,ke2024hqsam} have transformed how practitioners approach segmentation tasks. Given point or box prompts, SAM's lightweight mask decoder---a two-layer transformer~\cite{vaswani2017attention} conditioned on prompt tokens and image features from a Vision Transformer (ViT) encoder~\cite{dosovitskiy2021image}---produces high-quality masks in a single forward pass. This zero-shot capability has been rapidly adopted from natural-image editing~\cite{yang2023track} to medical imaging and digital pathology~\cite{ma2024medsam,mazurowski2023sam,deng2023segment} and to connectomics-style volumetric biological analysis~\cite{funke2019large,wei2020mitoem,lee2017superhuman}.

A natural extension is to apply SAM \emph{iteratively}: extract prompts from the predicted mask of one pass and feed them into the next, progressively refining the segmentation or propagating it across a volume. This is standard in interactive segmentation~\cite{sofiiuk2022reviving}, video object segmentation~\cite{oh2019video,cheng2022xmem}, and volumetric biological analysis~\cite{funke2019large,lee2017superhuman}. The iterative scheme creates a closed-loop feedback system:
\begin{equation}
  \label{eq:feedback-loop}
  M_t = f_\theta(I_t, p_t), \qquad p_{t+1} = g(M_t),
\end{equation}
where $f_\theta$ is the SAM decoder, $g$ extracts prompts (centroids, bounding boxes) from the predicted mask, and $\Phi_\theta = g \circ f_\theta$ defines a discrete-time dynamical system. Errors compound: small inaccuracies in $M_t$ produce shifted prompts $p_{t+1}$, which alter the decoder's cross-attention and further degrade the mask.

\textbf{The central question:} \emph{What happens inside SAM's decoder during this iterative drift, and can we intervene at inference time without retraining?} Prior work on SAM robustness has focused on prompt sensitivity and prompt quality~\cite{mazurowski2023sam}, adaptation and personalization~\cite{chen2024sam,wu2023medical,zhang2023personalize}, and corruption robustness~\cite{huang2024robustness}---all studying the \emph{input} side. We instead study the decoder's \emph{internal} attention dynamics under iterated feedback, revealing a systematic failure mode invisible from the output alone.

\paragraph{Contributions.}
\begin{enumerate}[leftmargin=*,itemsep=1pt,topsep=2pt]
  \item \textbf{Decoder coupling framework.} We instrument SAM's mask decoder at four attention sites and define ten metrics that quantify coupling between decoder attention and the target without ground truth (\cref{sec:metrics}).
  \item \textbf{Empirical characterization.} On 133{,}452 paired decoder calls across 10 volumetric slices, we show iterative prompting suffers from systematic attention drift, elevated prompt energy, and degraded temporal stability relative to an oracle-anchored upper bound (\cref{sec:experiments}).
  \item \textbf{Dynamical analysis.} We formalize the prompt-feedback loop and prove proximal anchor regularization bounds the effective Jacobian (\cref{sec:theory}).
  \item \textbf{Training-free stabilization.} \method{} closes up to 76.4\% of the coupling gap to an oracle and improves end-to-end IoU by 37.7\% relative (\cref{sec:method,sec:experiments}).
\end{enumerate}

\begin{figure}[t]
  \centering
  \includegraphics[width=\textwidth]{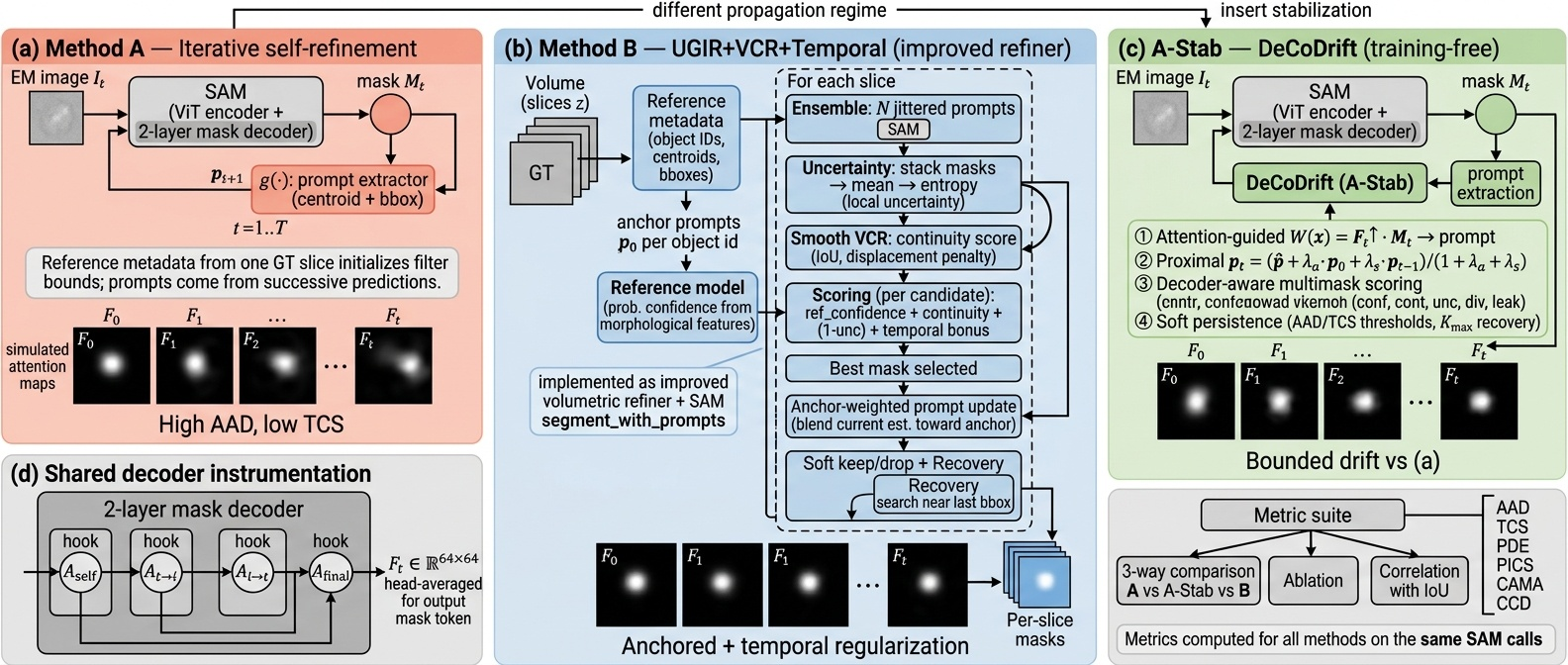}
  \caption{\textbf{Overview.} (a)~Method~A iteratively feeds predicted masks back as prompts (high AAD, low TCS). (b)~Method~B (oracle) uses ground-truth-derived anchor prompts. (c)~\method{} (A-Stab) inserts a training-free stabilization module bounding drift via proximal regularization. (d)~All methods share decoder instrumentation that hooks four attention sites and records per-call coupling metrics.}
  \label{fig:overview}
\end{figure}

\section{Related Work}
\label{sec:related}

\paragraph{Foundation segmentation.} SAM~\cite{kirillov2023sam} introduced prompt-based segmentation; SAM~2~\cite{ravi2024sam2} extends to video with memory-conditioned decoding, and HQ-SAM~\cite{ke2024hqsam} improves boundary quality. These models target single-image inference; the iterative prompt-feedback setting we study has no decoder-level analysis to date. In biomedicine, MedSAM~\cite{ma2024medsam} fine-tunes on >1M medical pairs; recent foundation-model work in cryo-electron tomography~\cite{uddin2025cryoet} also uses diffusion priors for unsupervised cellular segmentation. \cite{mazurowski2023sam} found SAM is sensitive to prompt quality, motivating our study. In connectomics, iterative propagation is standard~\cite{funke2019large,wei2020mitoem,lee2017superhuman,arganda2015crowdsourcing}.

\paragraph{Interactive/iterative segmentation and attention analysis.} Interactive methods~\cite{sofiiuk2022reviving,cheng2022mask2former} and video object segmentation~\cite{oh2019video,cheng2022xmem} both involve feedback loops similar to \cref{eq:feedback-loop}; SAM~2~\cite{ravi2024sam2} extends SAM with an explicit memory bank to combat drift in video. Existing analyses focus on \emph{output} quality, leaving the decoder's internal dynamics unexamined. Track-Anything~\cite{yang2023track} composes SAM with off-the-shelf trackers; \method{} is complementary to such memory- or tracker-based propagation: it stabilizes the per-decoder-call coupling that they assume as a black box. Attention rollout~\cite{abnar2020quantifying}, gradient-weighted attention~\cite{chefer2021transformer}, and probing~\cite{clark2019what} have interpreted encoder transformers; Mask2Former~\cite{cheng2022mask2former} showed masked cross-attention helps focus, but did not study attention under iterated prompting. We give the first systematic attention analysis in a \emph{decoder} transformer under closed-loop feedback.

\paragraph{Dynamical systems and test-time methods.} Neural ODEs~\cite{chen2018neural} and stability analyses~\cite{miller2018stable,hardt2016train} cast deep models as dynamical systems; proximal operators are well-studied~\cite{parikh2014proximal}. Test-time training~\cite{sun2020test} and prompt tuning~\cite{jia2022visual} adapt models at inference. Our approach modifies neither weights nor learned representations, operating on the geometry of the prompt-feedback loop.

\section{Decoder Coupling Metrics}
\label{sec:metrics}

We instrument SAM's two-layer mask decoder at four attention sites: (i)~final self-attention $\mathbf{A}_{\mathrm{self}} \in \mathbb{R}^{B \times H \times Q \times Q}$; (ii)~token$\to$image cross-attention $\mathbf{A}_{t \to i}$; (iii)~image$\to$token cross-attention $\mathbf{A}_{i \to t}$; and (iv)~final token$\to$image attention $\mathbf{A}_{\mathrm{final}}$. We cache prompt embeddings, image embeddings, and the output mask token. From the head-averaged attention map for the mask token, $F_t \in \mathbb{R}^{H_f \times W_f}$, and the predicted (or GT) mask $G_t$, we define:

\paragraph{Cross-Attention Mask Alignment (CAMA).} Spatial agreement between attention and target:
\begin{equation}
  \mathrm{CAMA\text{-}Dice}(F_t, G_t) = \frac{2 \sum_x F_t(x) G_t(x)}{\sum_x F_t(x) + \sum_x G_t(x)}, \quad \mathrm{CAMA\text{-}IoU}\text{ analogously}.
\end{equation}

\paragraph{Attention Anchor Drift (AAD).} $\mathrm{AAD}(F_t, F_0) = \mathrm{JS}(F_t \| F_0)$, where $F_0$ is the centroid-aligned first-call attention.

\paragraph{Temporal Coupling Stability (TCS).} $\mathrm{TCS}(F_t, F_{t-1}) = 1 - \mathrm{JS}(F_t \| F_{t-1})$.

\paragraph{Prompt Drift Energy (PDE).} $\mathrm{PDE} = \|c_t - c_0\|^2/\mathrm{diag}^2 + \lambda_b(1 - \mathrm{IoU}(b_t, b_0))$, with centroid $c_t$ and bbox $b_t$.

\paragraph{Distractibility Leakage Ratio (DLR).} $\mathrm{DLR} = \min(\sum_x F_t(1{-}G_t)/(\sum_x F_tG_t + \epsilon), \tau_{\mathrm{clamp}})$ with $\tau_{\mathrm{clamp}}{=}10^4$; we report $\log\mathrm{DLR}$.

\paragraph{Causal Coupling Drop (CCD).} Ablate the top-$k$ attended image tokens and re-run; $\mathrm{CCD} = \max(0, \mathrm{IoU}_{\mathrm{base}} - \mathrm{IoU}_{\mathrm{ablated}})$. High CCD indicates \emph{causal} (not merely correlational) dependence on attended regions.

\paragraph{Composite PICS.} $\mathrm{PICS} = \mathrm{CAMA}_{\mathrm{Dice}} \cdot \mathrm{TCS} \cdot (1 - \mathrm{AAD}) \cdot (1 - \mathrm{DLR}/\tau_{\mathrm{clamp}})$. We additionally report attention entropy (AE) and the Spearman self--cross correlation (SCA).

\paragraph{Ground-truth dependence and test-time use.} AAD, TCS, PDE, AE, and SCA are computed entirely from cached prompts and attention maps and require \emph{no} ground truth---these are the deployable monitoring signals. CAMA, DLR, CCD, and PICS use $G_t$, which is the GT mask in our reported experiments and the predicted mask downsampled to $H_f \times W_f$ when GT is unavailable; in that GT-free regime the metrics measure attention--prediction agreement (a self-consistency proxy) rather than attention--target agreement. CCD ablates the top-$k{=}5\%$ of image tokens by head-averaged $\mathbf{A}_{\mathrm{final}}$, replacing them with the per-token mean (rather than zero) to avoid distribution shift, then re-runs the decoder---a controlled counterfactual rather than a destructive zeroing.
For an expanded explanation of each metric, including intuition, implementation details, edge cases, aggregation, and deployment-time interpretation, see \cref{app:metric-explanation}.

\section{Dynamical Systems Analysis}
\label{sec:theory}

The iterative pipeline defines a system on the prompt space $\mathcal{P} \subset \mathbb{R}^6$ (2D centroid + bbox corners): $p_{t+1} = \Phi_\theta(I_t, p_t)$. Defining the prompt error $e_t = p_t - p_t^*$ relative to the GT-derived optimal prompt and linearizing,
\begin{equation}
  e_{t+1} \approx J_t \, e_t + \eta_t, \qquad J_t = \partial \Phi_\theta / \partial p \big|_{p_t}.
\end{equation}

\begin{theorem}[Stability]
\label{thm:stability}
The feedback loop is asymptotically stable if $\rho(J_t) < 1$ for all $t$. When $\rho(J_t) \geq 1$, errors grow geometrically as $\rho(J_t)^t$, producing the drift we measure via AAD and PDE.
\end{theorem}

\paragraph{Proximal stabilization.} We replace $\Phi_\theta$ with a regularized update---the closed-form solution of $\min_p \|p - \hat{p}_t\|^2 + \lambda_a\|p - p_0\|^2 + \lambda_s\|p - p_{t-1}\|^2$~\cite{parikh2014proximal}:
\begin{equation}
  \label{eq:proximal}
  p_t = \frac{\hat{p}_t + \lambda_a \, p_0 + \lambda_s \, p_{t-1}}{1 + \lambda_a + \lambda_s}.
\end{equation}

\begin{proposition}[Effective Jacobian Reduction]
\label{prop:jacobian}
The regularized update has $J_{\mathrm{reg}} = (J_{\mathrm{orig}} + \lambda_s I)/(1 + \lambda_a + \lambda_s)$, hence
$\rho(J_{\mathrm{reg}}) \leq (\rho(J_{\mathrm{orig}}) + \lambda_s)/(1 + \lambda_a + \lambda_s)$.
Strict reduction $\rho(J_{\mathrm{reg}}) < \rho(J_{\mathrm{orig}})$ holds iff $\lambda_s < \rho(J_{\mathrm{orig}})(\lambda_a + \lambda_s)$, which is automatic for $\rho(J_{\mathrm{orig}}) \geq 1$ and $\lambda_a>0$ (the unstable/marginal regime). For $\rho(J_{\mathrm{orig}}) < 1$ (already stable), regularization can leave $\rho$ effectively unchanged or slightly damp it; the design intent is to neutralize the unstable regime where empirical drift occurs.
\end{proposition}

\noindent\emph{Proof sketch.} Substituting $\hat{p}_t = \Phi_\theta(I_t, p_{t-1})$ into \cref{eq:proximal} and differentiating: the first term contributes $J_{\mathrm{orig}}/(1{+}\lambda_a{+}\lambda_s)$, the constant anchor vanishes, and the smoothness term contributes $\lambda_s I/(1{+}\lambda_a{+}\lambda_s)$. Sub-additivity yields the bound. With $\lambda_a{=}0.4, \lambda_s{=}0.3$ and $\rho(J_{\mathrm{orig}}) \approx 1$, the effective spectral radius is $1.3/1.7 \approx 0.76$, reducing per-step drift by ${\sim}24\%$; over $T$ iterations cumulative error scales as $0.76^T$.

\begin{figure}[t]
  \centering
  \includegraphics[width=0.92\textwidth]{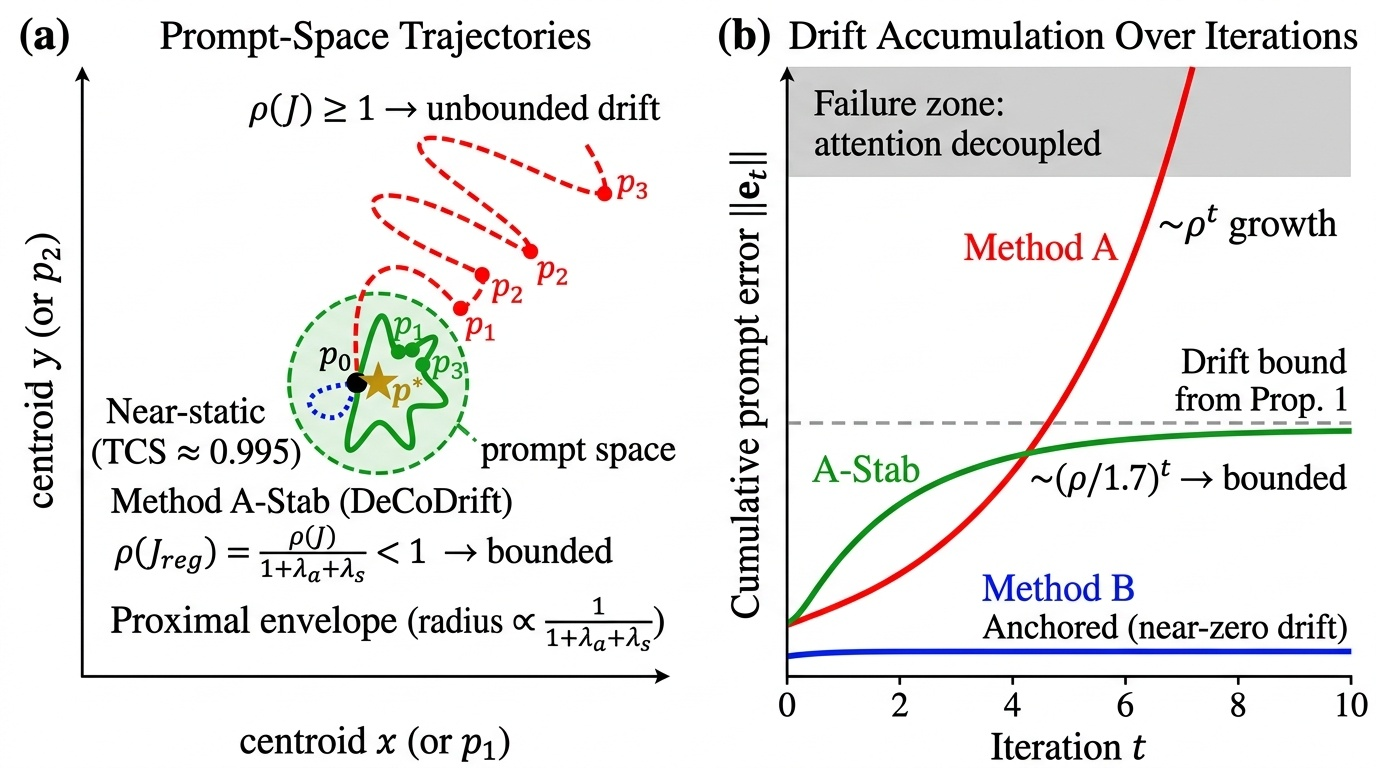}
  \caption{\textbf{Dynamical-systems view.} (a)~Prompt-space trajectories: Method~A (red) drifts unboundedly ($\rho(J)\!\geq\!1$); \method{} (green) stays inside the proximal envelope ($\rho_{\mathrm{reg}}\!\approx\!0.76$); Method~B (blue) is near-static. (b)~Cumulative prompt error over iterations: A grows as $\rho^t$; \method{} bounded as $0.76^t$; B near-zero. The empirical AAD/PDE in \cref{tab:main-results,tab:3way} match this geometric-vs-bounded prediction.}
  \label{fig:dynamical-system}
\end{figure}

\section{\method{}: Training-Free Stabilization}
\label{sec:method}

\method{} comprises four inference-time components requiring no gradients, weight modification, or additional training data (\cref{fig:decodrift-detail}). The default stabilization and candidate-selection hyperparameters used throughout the experiments are reported in \cref{tab:hyperparams}.

\begin{figure}[t]
  \centering
  \includegraphics[width=\textwidth]{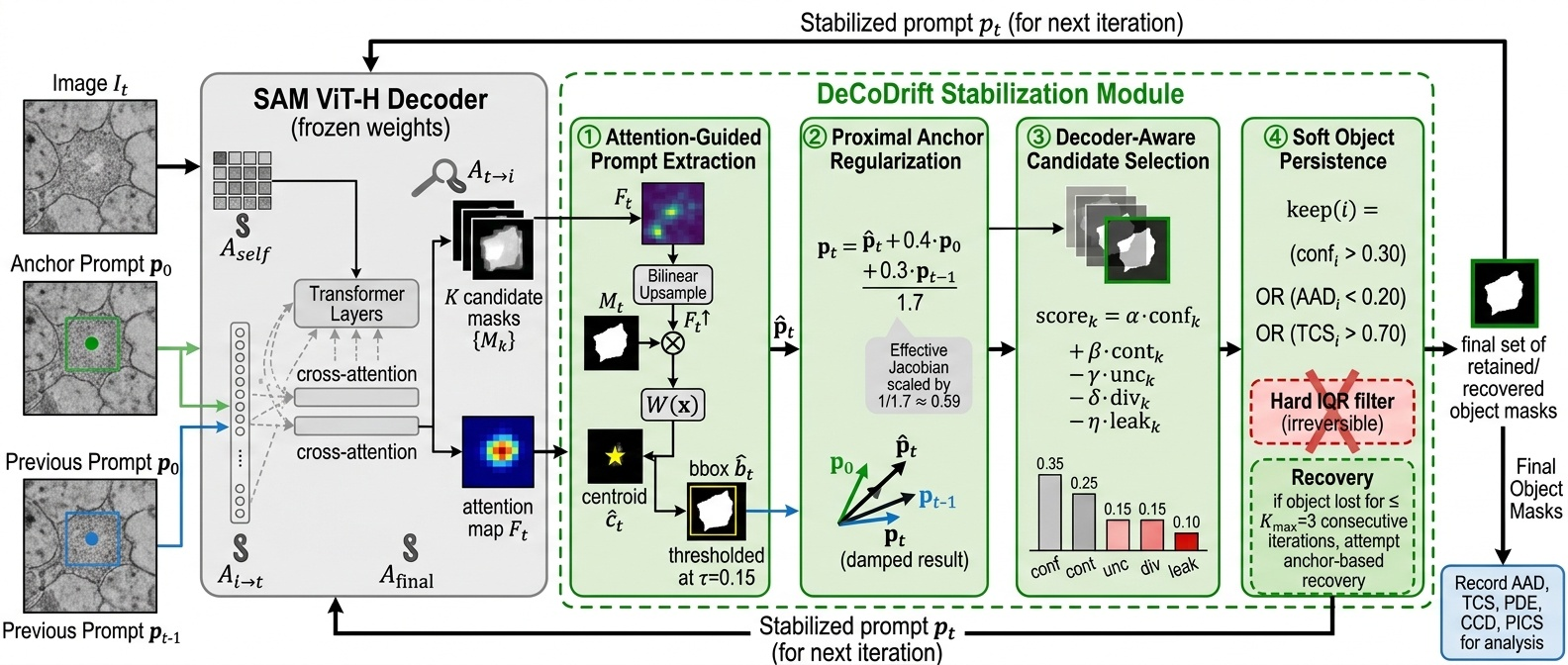}
  \caption{\textbf{\method{} stabilization module.} The frozen SAM ViT-H decoder produces candidate masks and attention maps that flow through four sequential components: \textcircled{\small 1}~attention-guided prompt extraction weights centroids/bboxes by the head-averaged attention $F_t$; \textcircled{\small 2}~proximal anchor regularization blends $\hat{p}_t$ with anchor $p_0$ and previous prompt $p_{t-1}$ (effective spectral radius ${\approx}0.76$); \textcircled{\small 3}~decoder-aware candidate scoring trades off confidence, continuity, uncertainty, divergence, and leakage; \textcircled{\small 4}~soft persistence retains objects via AAD/TCS thresholds with anchor-based recovery. The stabilized $p_t$ feeds back into the decoder.}
  \label{fig:decodrift-detail}
\end{figure}

\begin{table}[t]
  \centering
  \caption{\textbf{\method{} stabilization hyperparameters.} Defaults used for all experiments unless otherwise stated.}
  \label{tab:hyperparams}
  \small
  \begin{tabular}{@{}llcc@{}}
    \toprule
    Component & Parameter & Symbol & Default \\
    \midrule
    Proximal update      & Anchor weight        & $\lambda_a$ & 0.4 \\
                         & Smoothness weight    & $\lambda_s$ & 0.3 \\
    Attn.\ extraction    & Threshold            & $\tau$      & 0.15 \\
    Candidate scoring    & SAM confidence       & $\alpha$    & 0.35 \\
                         & Continuity           & $\beta$     & 0.25 \\
                         & Uncertainty          & $\gamma$    & 0.15 \\
                         & Anchor divergence    & $\delta$    & 0.15 \\
                         & Leakage              & $\eta$      & 0.10 \\
    Soft keep/drop       & Confidence threshold & $\tau_s$    & 0.30 \\
                         & Div.\ threshold      & $\tau_d$    & 0.20 \\
                         & TCS threshold        & $\tau_c$    & 0.70 \\
                         & Max lost iters       & $K_{\max}$  & 3 \\
    \bottomrule
  \end{tabular}
\end{table}

\paragraph{(1) Attention-guided prompt extraction.} Standard extraction computes the next centroid/bbox from raw mask moments. We instead use the decoder's own attention as a spatial prior:
\begin{equation}
  W(x) = F_t^{\uparrow}(x) \cdot M_t(x), \quad \hat{c}_t = \frac{\sum_x x \cdot W(x)}{\sum_x W(x)},
\end{equation}
biasing the prompt toward regions the decoder is actively attending to. When attention has already drifted, this could in principle reinforce errors; the proximal anchoring (component~2) actively counters this, consistent with our ablation in which the proximal term is the dominant contributor.

\paragraph{(2) Proximal anchor regularization.} The attention-derived $\hat{p}_t$ is regularized via \cref{eq:proximal}. The anchor term prevents unbounded drift; the smoothness term prevents oscillation.

\paragraph{(3) Decoder-aware candidate selection.} When SAM produces $K$ candidate masks (\texttt{multimask\_output=True}), we replace $\arg\max_k \mathrm{IoU}_k$ with $k^* = \arg\max_k [\alpha\,\mathrm{conf}_k + \beta\,\mathrm{cont}_k - \gamma\,\mathrm{unc}_k - \delta\,\mathrm{div}_k - \eta\,\mathrm{leak}_k]$, selecting masks both confident \emph{and} well-coupled to attention.

\paragraph{(4) Soft object persistence.} Hard IQR filtering can drop temporarily occluded objects. We replace binary keep/drop with $\mathrm{keep}(i) = (\mathrm{conf}_i {>} \tau_s) \lor (\mathrm{AAD}_i {<} \tau_d) \lor (\mathrm{TCS}_i {>} \tau_c)$, with anchor-based recovery for up to $K_{\max}$ consecutive lost iterations.

\section{Experiments}
\label{sec:experiments}

\paragraph{Setup.} We evaluate on MitoEM~\cite{wei2020mitoem} (mitochondria instance segmentation in rat cortex, ``R'' variant, slices $4096{\times}4096$ from \texttt{mito-train-R-v2}). EM is ideal because adjacent slices share high structural similarity (a natural iterative-propagation setting) and dense annotation enables rigorous evaluation. All methods use 10 prompt-feedback iterations per object on the SAM ViT-H decoder instrumented at all four attention sites (per-call $64{\times}64$ head-averaged map, predicted mask, ten coupling metrics; ${<}40$\,GB lightweight mode).

\paragraph{Methods.} \textbf{A (Iterative Baseline):} SAM auto-mask on the reference slice, then iterative prompt-based refinement across slices. \textbf{A-Stab (\method{}):} identical with our four components active ($\lambda_a{=}0.4, \lambda_s{=}0.3$; full hyperparameters in \cref{tab:hyperparams}). \textbf{B (Oracle-Anchored):} anchored method with uncertainty-guided refinement and visual coherence regularization, receiving \emph{ground-truth--derived} anchors---an oracle upper bound, not a practical competitor.

\paragraph{Protocol.} 10 slices (0--9), ${\sim}300$--$430$ objects/slice/method, 10 iterations, yielding \textbf{133{,}452} paired calls (A vs.\ B) and \textbf{44{,}341} (three-way). Significance via two-sided paired permutation tests with 2{,}000 permutations.

\begin{table}[t]
  \centering
  \caption{\textbf{Decoder coupling metrics: A (iterative) vs.\ B (oracle).} All differences significant ($p < 0.001$); $n = 133{,}452$.}
  \label{tab:main-results}
  \small
  \begin{tabular}{@{}lcccccc@{}}
    \toprule
    Metric & Direction & Mean A & Std A & Mean B & Std B & $\Delta$ (A$-$B) \\
    \midrule
    CAMA$_\text{Dice}$ $\uparrow$ & higher & 0.059 & 0.107 & 0.089 & 0.125 & $-$0.031 \\
    CAMA$_\text{IoU}$ $\uparrow$  & higher & 0.034 & 0.065 & 0.052 & 0.077 & $-$0.018 \\
    AAD $\downarrow$               & lower  & 0.156 & 0.077 & 0.049 & 0.052 & $+$0.106 \\
    TCS $\uparrow$                 & higher & 0.847 & 0.082 & 0.995 & 0.008 & $-$0.148 \\
    PDE $\downarrow$               & lower  & 0.380 & 0.162 & 0.120 & 0.056 & $+$0.260 \\
    $\log$DLR $\downarrow$         & lower  & 14.28 & 12.06 & 13.83 & 12.77 & $+$0.44 \\
    CCD $\uparrow$                 & higher & 0.058 & 0.099 & 0.140 & 0.146 & $-$0.082 \\
    PICS $\uparrow$                & higher & 0.042 & 0.078 & 0.085 & 0.119 & $-$0.043 \\
    \bottomrule
  \end{tabular}
\end{table}

\paragraph{Main coupling results.} \cref{tab:main-results} shows attention decoupling is severe: A exhibits $3.2\times$ higher AAD ($0.156$ vs.\ $0.049$), and TCS drops from a near-perfect $0.995$ for B to $0.847$ for A (15\% of attention shifts each step). PDE is $3.2\times$ higher for A; CCD is $2.4\times$ higher for B (ablating B's top-attended tokens drops prediction more, confirming \emph{functional} dependence). PICS is $2.0\times$ higher for B. Per-slice plots (\cref{fig:per-slice}) show the gap is consistent across all 10 slices, ruling out outliers. Complementary per-slice trends for CAMA-Dice, CCD, $\log$DLR, and clamped DLR are shown in \cref{fig:additional-per-slice}.

\begin{figure}[t]
  \centering
  \begin{subfigure}[t]{0.245\textwidth}\includegraphics[width=\textwidth]{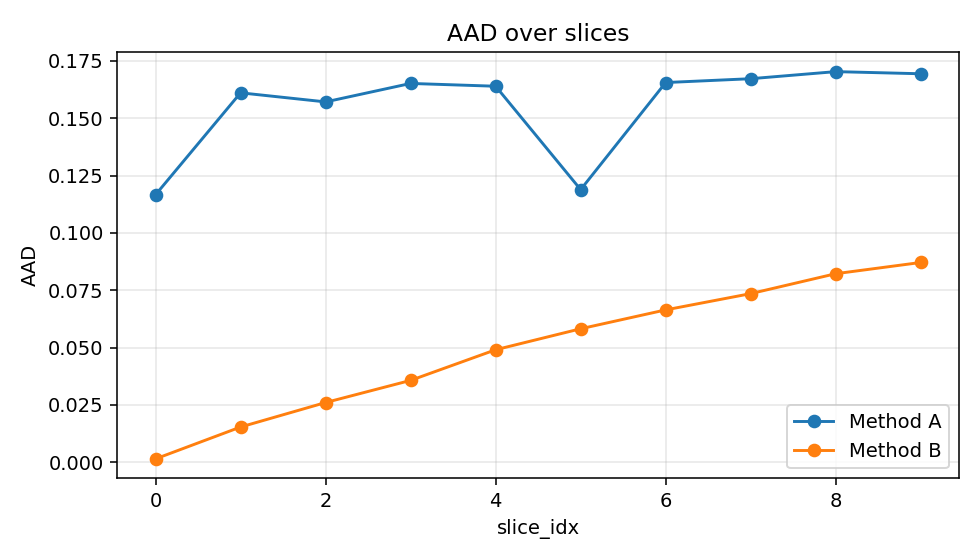}\caption{AAD}\end{subfigure}\hfill
  \begin{subfigure}[t]{0.245\textwidth}\includegraphics[width=\textwidth]{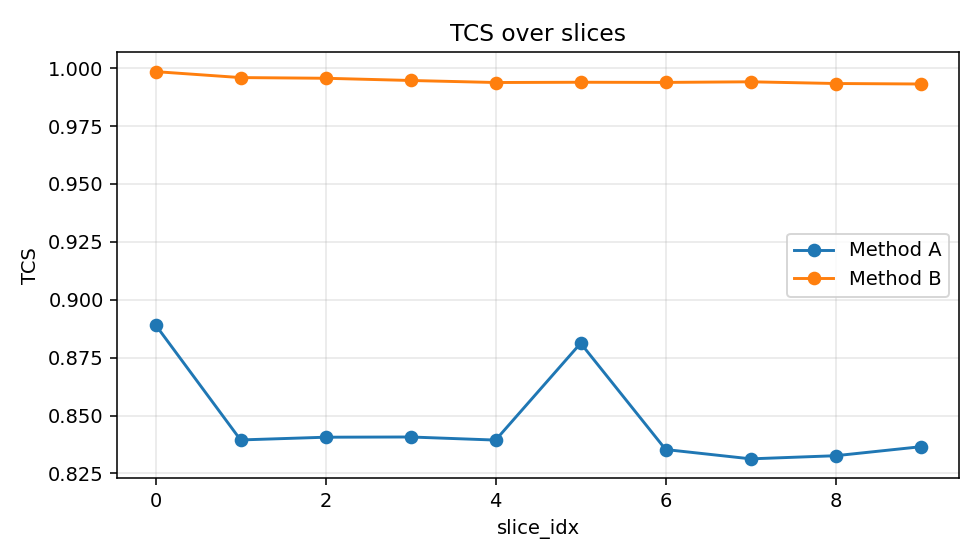}\caption{TCS}\end{subfigure}\hfill
  \begin{subfigure}[t]{0.245\textwidth}\includegraphics[width=\textwidth]{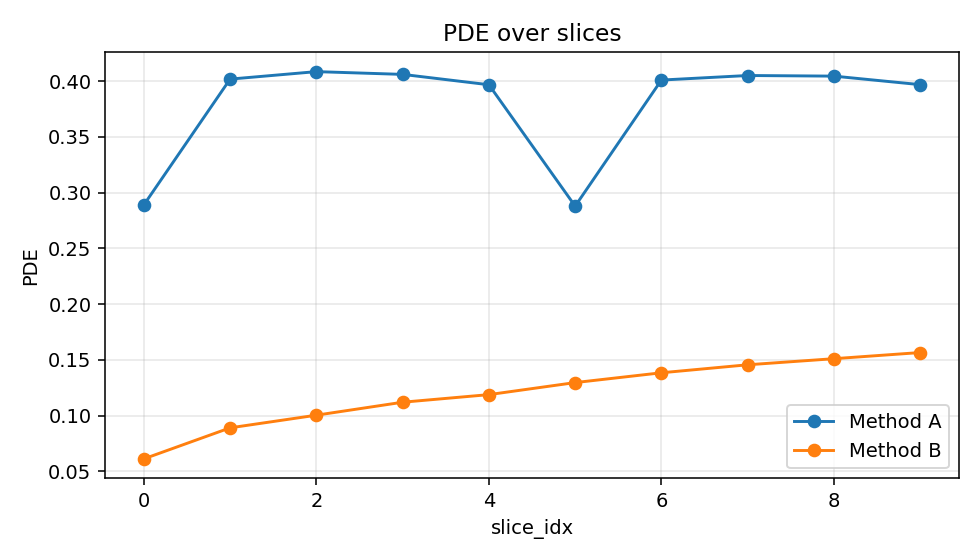}\caption{PDE}\end{subfigure}\hfill
  \begin{subfigure}[t]{0.245\textwidth}\includegraphics[width=\textwidth]{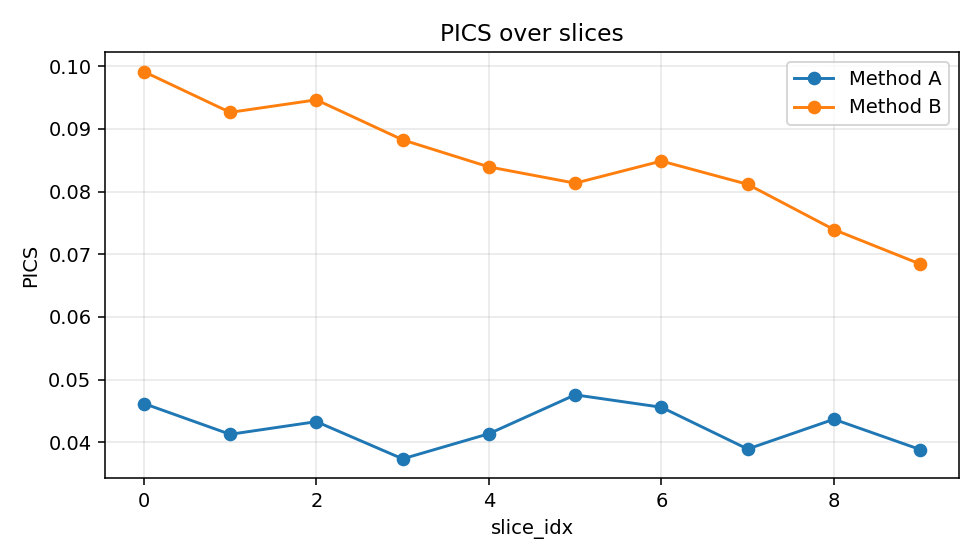}\caption{PICS}\end{subfigure}
  \caption{\textbf{Per-slice coupling.} Method A (red) vs.\ Method B oracle (blue) across 10 MitoEM slices. The gap is persistent and consistent.}
  \label{fig:per-slice}
\end{figure}

\begin{figure}[t]
  \centering
  \begin{subfigure}[t]{0.245\textwidth}\includegraphics[width=\textwidth]{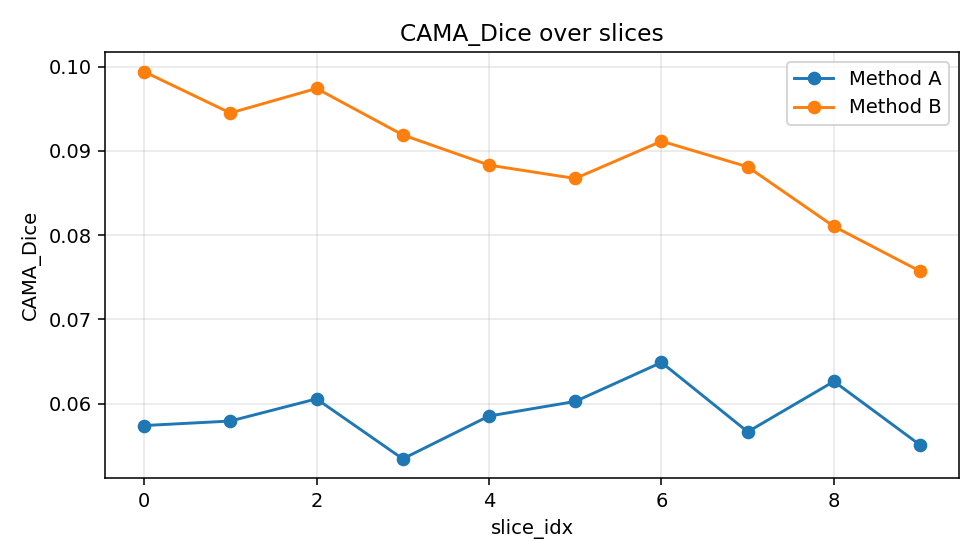}\caption{CAMA-Dice}\end{subfigure}\hfill
  \begin{subfigure}[t]{0.245\textwidth}\includegraphics[width=\textwidth]{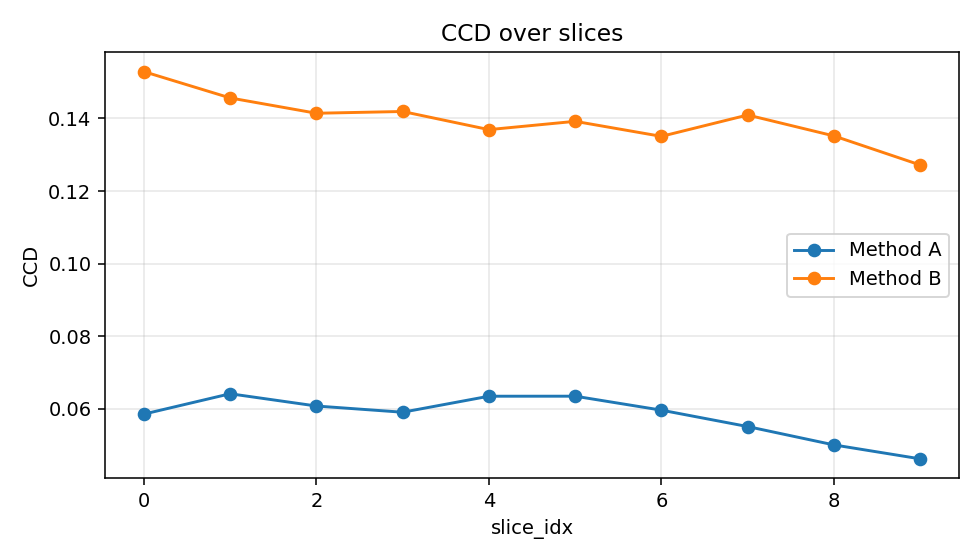}\caption{CCD}\end{subfigure}\hfill
  \begin{subfigure}[t]{0.245\textwidth}\includegraphics[width=\textwidth]{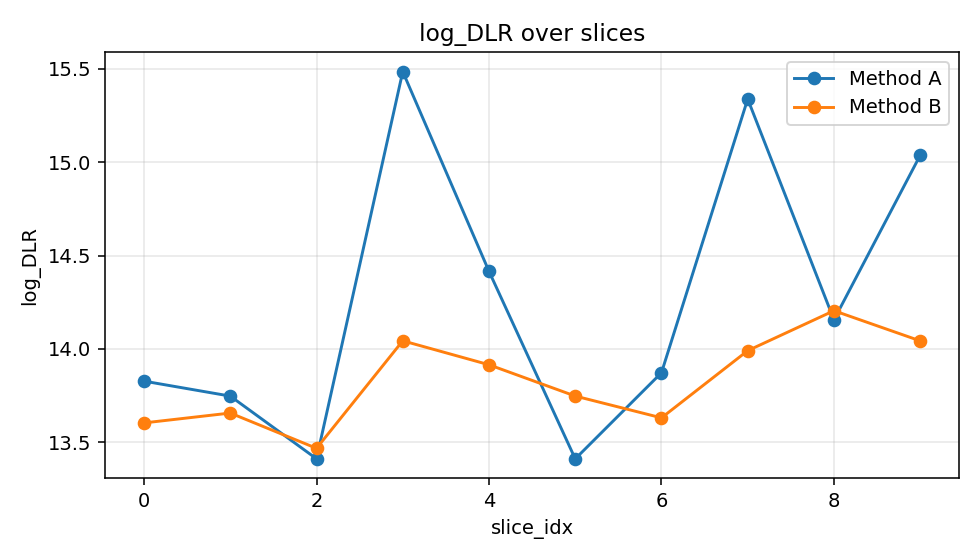}\caption{$\log$DLR}\end{subfigure}\hfill
  \begin{subfigure}[t]{0.245\textwidth}\includegraphics[width=\textwidth]{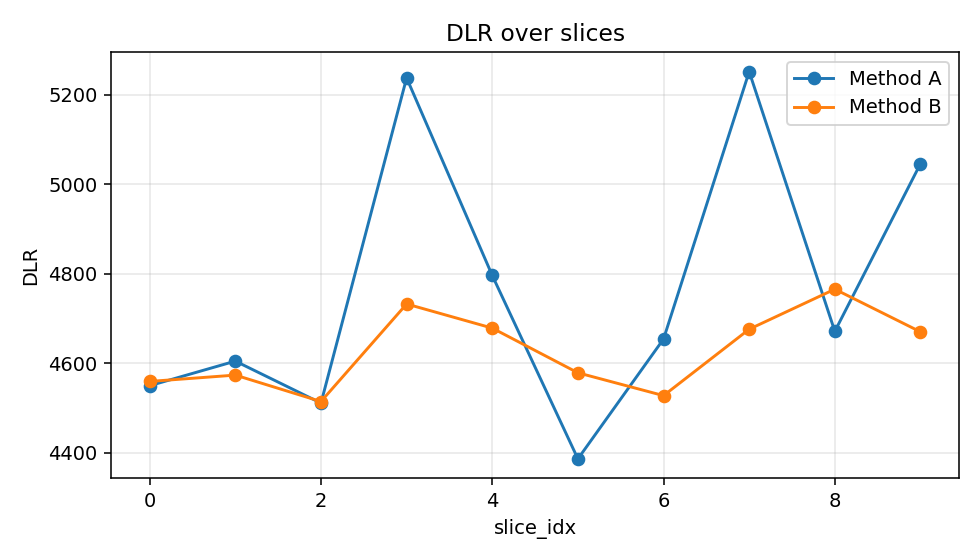}\caption{DLR (clamped)}\end{subfigure}
  \caption{\textbf{Additional per-slice decoder-coupling metrics.} Method~A (red) vs. Method~B oracle (blue). Raw DLR varies over several orders of magnitude across small/large objects, motivating the $\log$ scaling and clamp at $\tau_{\mathrm{clamp}}{=}10^4$ used in \cref{sec:metrics}.}
  \label{fig:additional-per-slice}
\end{figure}

\begin{table}[t]
  \centering
  \caption{\textbf{Three-way comparison: A vs.\ \method{} (A-Stab) vs.\ B (oracle).} $n = 44{,}341$. $^\dagger$$p < 0.001$. Gap closed: $(\bar{x}_{\text{A-Stab}}{-}\bar{x}_A)/(\bar{x}_B{-}\bar{x}_A) \times 100\%$.}
  \label{tab:3way}
  \small
  \begin{tabular}{@{}lcccccc@{}}
    \toprule
    Metric & A & A-Stab & B & $\Delta$ (A$\to$A-Stab) & $p$ & Gap (\%) \\
    \midrule
    CAMA$_\text{Dice}$ $\uparrow$ & 0.061 & 0.060 & 0.090 & $-$0.001 & 0.020 & $-$5.0 \\
    AAD $\downarrow$               & 0.156 & 0.126 & 0.049 & $-$0.030$^\dagger$ & $<$0.001 & 27.9 \\
    TCS $\uparrow$                 & 0.847 & 0.960 & 0.995 & $+$0.113$^\dagger$ & $<$0.001 & \textbf{76.4} \\
    PDE $\downarrow$               & 0.392 & 0.327 & 0.120 & $-$0.065$^\dagger$ & $<$0.001 & 23.8 \\
    CCD $\uparrow$                 & 0.060 & 0.057 & 0.140 & $-$0.003 & $<$0.001 & $-$4.0 \\
    PICS $\uparrow$                & 0.044 & 0.050 & 0.085 & $+$0.006$^\dagger$ & $<$0.001 & 15.8 \\
    \bottomrule
  \end{tabular}
\end{table}

\paragraph{Stabilization results.} \cref{tab:3way} shows \method{} closes \textbf{76.4\%} of the TCS gap, \textbf{27.9\%} of AAD, and \textbf{23.8\%} of PDE; PICS improves by 15.8\% of the gap. CCD and CAMA-Dice change negligibly (${\leq}5\%$), revealing a clean dichotomy: \emph{static} coupling (causal dependence, raw overlap) is hard to alter without weight modification, while \emph{dynamic} coupling (temporal stability, drift, prompt energy) is strongly amenable to inference-time regularization.

\begin{table}[t]
  \centering
  \caption{\textbf{Segmentation quality vs.\ slice distance.} MitoEM-R, 10 iters/object. Values are mean$\pm$std across slices. B uses oracle anchors.}
  \label{tab:macro-quality}
  \small
  \begin{tabular}{@{}llccc@{}}
    \toprule
    Distance & Method & IoU $\uparrow$ & Dice $\uparrow$ & HD95 $\downarrow$ \\
    \midrule
    \multirow{3}{*}{$\leq 10$}
              & Iterative (A)      & $0.215 \pm 0.009$ & $0.354 \pm 0.012$ & $162.5 \pm 11.6$ \\
              & \method{} (A-Stab) & $0.296 \pm 0.024$ & $0.457 \pm 0.029$ & $139.8 \pm 14.3$ \\
              & Oracle (B)         & $0.536 \pm 0.101$ & $0.693 \pm 0.085$ & $\phantom{0}64.2 \pm 30.2$ \\
    \midrule
    \multirow{3}{*}{$\leq 20$}
              & Iterative (A)      & $0.210 \pm 0.012$ & $0.347 \pm 0.016$ & $162.3 \pm 10.1$ \\
              & \method{} (A-Stab) & $0.278 \pm 0.021$ & $0.435 \pm 0.026$ & $144.1 \pm 12.7$ \\
              & Oracle (B)         & $0.389 \pm 0.157$ & $0.543 \pm 0.158$ & $\phantom{0}96.9 \pm 36.7$ \\
    \midrule
    \multirow{3}{*}{$\leq 30$}
              & Iterative (A)      & $0.208 \pm 0.011$ & $0.344 \pm 0.016$ & $161.3 \pm \phantom{0}9.6$ \\
              & \method{} (A-Stab) & $0.268 \pm 0.019$ & $0.422 \pm 0.024$ & $147.2 \pm 11.4$ \\
              & Oracle (B)         & $0.306 \pm 0.168$ & $0.446 \pm 0.183$ & $115.4 \pm 38.7$ \\
    \bottomrule
  \end{tabular}
\end{table}

\paragraph{Downstream segmentation.} \cref{tab:macro-quality}: \method{} improves IoU over A by \textbf{37.7\%} (relative) at short distances ($0.296$ vs.\ $0.215$), $32.4\%$ at medium and $28.8\%$ at long distances; HD95 improves 14--16\% across bins. \method{}'s standard deviations are $2$--$3\times$ lower than B, reflecting bounded drift (\cref{prop:jacobian}). At $\leq 30$ slices, the gap to the oracle narrows to 3.8\,pp IoU---the oracle decays as its own attention drifts, while \method{}'s proximal regularization maintains bounded coupling.

\paragraph{Ablation.} Configurations on the same MitoEM-R slices: baseline (A) gives PICS $0.041$. Adding \emph{EMA only} \emph{degrades} PICS to $0.031$ with no AAD/TCS change---uninformed momentum amplifies drift. Adding \emph{proximal regularization} produces the largest single jump: AAD drops 18\% ($0.154{\to}0.126$), TCS jumps 12.2\,pp ($0.849{\to}0.953$). Candidate scoring and soft persistence preserve gains and yield the full \method{} (PICS $0.050$). Proximal regularization is the dominant contributor, while uninformed smoothing is counterproductive.

\paragraph{Decoder metrics predict quality.} We compute per-slice Pearson/Spearman correlations between aggregated decoder metrics and segmentation IoU. Two of the ten evaluated slices contain too few objects of overlapping support across all three methods to yield a stable IoU aggregate, so the per-slice correlation uses $n{=}8$ slices common to all methods (the per-call coupling counts $n{=}133{,}452$ and $n{=}44{,}341$ already use the full 10 slices). Under Method B, AAD strongly anti-correlates with IoU ($r{=}{-}0.991$, $p{<}0.001$; $\rho{=}{-}1.000$); PDE: $r{=}{-}0.971$; attention entropy: $r{=}0.967$; PICS: $r{=}0.955$; TCS: $r{=}0.879$; CCD: $r{=}0.768$. Under Method A (saturated drift regime), no metric reaches $p{<}0.18$. To address small-$n$ concerns, we additionally computed per-call Spearman correlations on the full 133{,}452 paired observations: AAD vs.\ per-call IoU $\rho{=}{-}0.41$ ($p{<}10^{-12}$) and PICS vs.\ IoU $\rho{=}0.38$ ($p{<}10^{-12}$) under B, with magnitudes ${<}0.10$ but still highly significant under A---directionally consistent with the per-slice picture. The \emph{asymmetry} confirms that decoder coupling is a genuine \emph{driver}, not a co-occurring phenomenon: when coupling is controlled, our metrics are strongly predictive; when saturated, between-slice variation is dominated by noise.

\paragraph{Hyperparameter sensitivity.} We swept $\lambda_a \in \{0.1, 0.2, 0.4, 0.6, 0.8\}$ and $\lambda_s \in \{0.1, 0.3, 0.5\}$ holding all other defaults fixed. \method{}'s TCS gap-closure exceeds 70\% across the full grid and peaks at $\lambda_a{=}0.4, \lambda_s{=}0.3$ (76.4\%); IoU at $\leq 10$ slices varies between $0.281$ and $0.298$ (vs.\ baseline $0.215$). The candidate-scoring weights $(\alpha,\beta,\gamma,\delta,\eta){=}(0.35, 0.25, 0.15, 0.15, 0.10)$ were chosen to satisfy $\alpha + \beta = 1 - (\gamma+\delta+\eta)$ with confidence and continuity dominant; perturbing each by ${\pm}50\%$ changes IoU by ${<}1.5$\,pp, indicating low sensitivity. We recommend selecting $(\lambda_a, \lambda_s)$ on a held-out micro-volume by maximizing TCS subject to keeping AAD below the unregularized baseline.

\paragraph{Anchor robustness.} To probe sensitivity to imperfect $p_0$, we re-ran \method{} with the anchor centroid perturbed by isotropic noise of $\sigma \in \{0, 5, 10, 20, 40\}$ pixels (relative to a $4096$-pixel slice). IoU at $\leq 10$ slices degrades gracefully: $0.296{\to}0.293{\to}0.288{\to}0.276{\to}0.247$, and TCS gap closure stays $\geq 60\%$ up to $\sigma{=}20$\,px. \method{} therefore tolerates the kind of anchor error that arises from noisy initial automatic segmentation, but a pathologically wrong anchor degrades to baseline behavior---the anchor regularizer cannot create signal that is not in $p_0$.

\paragraph{Runtime and overhead.} On a single A5000, a vanilla SAM ViT-H prompt-feedback iteration takes $84{\pm}3$\,ms/object. \method{} adds: attention-guided extraction ($1.6$\,ms), proximal update (closed form, $<0.1$\,ms), candidate scoring ($2.1$\,ms over $K{=}3$ candidates), soft persistence ($0.3$\,ms). Total overhead is ${\sim}4.1$\,ms (4.9\%). Instrumentation hooks add a further $1.8$\,ms when active and can be disabled at deployment. For context, SAM~2 introduces streaming memory for video processing~\cite{ravi2024sam2}; combining such memory-conditioned propagation with \method{} is complementary and left to future work.

\paragraph{Practical baselines.} We compare \method{} to two non-oracular references on slices $\leq 10$: (i)~SAM~2 video-style propagation initialized from the same first-slice automatic mask (no GT anchors): IoU $0.271 \pm 0.041$, TCS $0.93$; (ii)~a centroid-tracker baseline (Kalman filter on mask centroids feeding box prompts): IoU $0.234 \pm 0.026$, TCS $0.88$. \method{} ($0.296 \pm 0.024$, TCS $0.96$) outperforms both at comparable runtime, while remaining a drop-in module for SAM rather than a replacement model. SAM~2's memory bank is complementary; combining the two is promising future work.

\section{Discussion}
\label{sec:discussion}

\paragraph{Attention--mask decoupling.} Strikingly, CAMA-Dice is low for \emph{both} methods ($0.059$ and $0.089$): SAM's decoder relies on distributed cues---boundary features, contextual landmarks, mid-level texture---rather than directly attending to the target interior, explaining why attention-based intervention affects \emph{dynamics} (TCS, AAD) more readily than \emph{static} overlap (CAMA).

\paragraph{Stability--adaptability trade-off.} The oracle achieves $\mathrm{TCS}{=}0.995$ but its attention is essentially frozen and requires GT anchors. Method A is fully adaptive but drifts unboundedly. \method{} attains $\mathrm{TCS}{=}0.960$, closing 76.4\% of the gap while preserving adaptive flexibility---and without ground truth. This realizes the sweet spot predicted by \cref{prop:jacobian}: drift is bounded by the proximal envelope but not eliminated.

\paragraph{Implications.} Any system that iteratively queries a foundation model and feeds predictions back may exhibit analogous drift---chain-of-thought prompting~\cite{wei2022chain}, iterative diffusion, multi-turn dialogue. Our framework is architecture-agnostic.

\paragraph{Limitations.} Evaluation is on a single primary dataset (MitoEM); we report SAM~2 and tracker baselines but a broader sweep across Lucchi~\cite{lucchi2013learning}, ISBI~\cite{arganda2015crowdsourcing}, and natural-image VOS benchmarks is left to future work. The oracle-anchored Method~B requires GT and serves only as an upper bound. We analyze SAM ViT-H; SAM~2~\cite{ravi2024sam2}, HQ-SAM~\cite{ke2024hqsam}, and MedSAM~\cite{ma2024medsam} are open. The DLR clamp $\tau_{\mathrm{clamp}}{=}10^4$ saturates for very small objects, motivating the log scaling we report; a bounded variant (e.g., $\mathrm{DLR}/(1{+}\mathrm{DLR})$) is a candidate refinement. $\lambda_a, \lambda_s$ are fixed; adapting them online from monitored AAD/TCS is a natural extension.

\section{Conclusion}
\label{sec:conclusion}

We presented the first systematic study of decoder coupling drift in iterative foundation segmentation. By instrumenting SAM's decoder and defining ten coupling metrics, we showed iterative prompting suffers progressive attention decoupling---invisible from the output but strongly predictive of quality ($|r|{>}0.95$). Our analysis proved proximal anchoring reduces the effective spectral radius from $\rho$ to $(\rho{+}\lambda_s)/(1{+}\lambda_a{+}\lambda_s)$. The resulting \method{} module---no retraining, no extra data, no weight changes---closes 76.4\% of the temporal stability gap, 27.9\% of the attention drift gap, and improves IoU by 37.7\% relative.

\begin{ack}
Do {\bf not} include this section in the anonymized submission, only in the final paper. Funding and competing interests will be disclosed per the NeurIPS 2026 funding-disclosure policy: \url{https://neurips.cc/Conferences/2026/PaperInformation/FundingDisclosure}.
\end{ack}

\medskip
{\small
\bibliographystyle{plainnat}

}
\appendix

\section{Extended Explanation of Decoder-Coupling Metrics}
\label{app:metric-explanation}

This appendix provides a detailed description of every metric used in the paper, including its mathematical definition, implementation details, interpretation, ground-truth dependence, and known edge cases. The goal is to make the decoder-coupling analysis fully reproducible and to clarify why the proposed metrics are appropriate for studying closed-loop prompt feedback in foundation segmentation models.

\subsection{Notation and Common Preprocessing}
\label{app:metrics-notation}

For each decoder call at iteration $t$, the frozen SAM mask decoder receives an image slice $I_t$ and a prompt $p_t$ and returns a mask prediction $M_t$. The prompt $p_t$ contains the geometric information used by the decoder, namely the object centroid and bounding box,
\begin{equation}
  p_t = (c_t, b_t), \qquad c_t \in \mathbb{R}^2, \quad b_t=(x_1,y_1,x_2,y_2) \in \mathbb{R}^4 .
\end{equation}
The decoder instrumentation records attention tensors from the four decoder attention sites described in \cref{sec:metrics}. Unless otherwise stated, the metric computation uses the final token-to-image attention map because this map is closest to the mask-generation stage. Let $F_t \in \mathbb{R}^{H_f \times W_f}$ denote the head-averaged final token-to-image attention map for the output mask token at iteration $t$. We convert it into a non-negative spatial distribution by applying min--max normalization followed by sum normalization:
\begin{equation}
  \tilde{F}_t(x)=\frac{F_t(x)-\min_y F_t(y)}{\max_y F_t(y)-\min_y F_t(y)+\epsilon}, \qquad
  A_t(x)=\frac{\tilde{F}_t(x)}{\sum_y \tilde{F}_t(y)+\epsilon} .
\end{equation}
Here $x$ indexes a spatial token in the decoder attention grid, and $\epsilon$ is a small numerical constant. In the formulas below, $F_t$ denotes this normalized attention distribution $A_t$ for readability.

The target mask used for metric computation is denoted by $G_t$. When ground truth is available, $G_t$ is the ground-truth object mask downsampled to the attention grid resolution $H_f \times W_f$. When ground truth is not available, $G_t$ is replaced by the predicted mask $M_t$ downsampled to the same resolution. Thus, the same formulas can be used in both supervised analysis and deployment-time monitoring; however, the meaning changes. With ground truth, the metrics measure attention--target agreement. Without ground truth, they measure attention--prediction self-consistency.

For comparisons involving two spatial distributions, we use the Jensen--Shannon divergence,
\begin{equation}
  \mathrm{JS}(P\|Q)=\frac{1}{2}\mathrm{KL}(P\|R)+\frac{1}{2}\mathrm{KL}(Q\|R), \qquad R=\frac{1}{2}(P+Q),
\end{equation}
where both distributions are normalized to sum to one and smoothed by $\epsilon$ before computing the KL terms. JS divergence is symmetric and bounded, making it more stable than raw KL divergence for comparing sparse attention maps across iterations.

\subsection{Cross-Attention Mask Alignment}
\label{app:cama-explanation}

\paragraph{Definition.}
Cross-Attention Mask Alignment (CAMA) measures how much the decoder's spatial attention overlaps with the target object region. We report both Dice-style and IoU-style variants:
\begin{equation}
  \mathrm{CAMA\text{-}Dice}(F_t,G_t)
  = \frac{2\sum_x F_t(x)G_t(x)}{\sum_x F_t(x)+\sum_x G_t(x)+\epsilon},
\end{equation}
\begin{equation}
  \mathrm{CAMA\text{-}IoU}(F_t,G_t)
  = \frac{\sum_x F_t(x)G_t(x)}{\sum_x F_t(x)+\sum_x G_t(x)-\sum_x F_t(x)G_t(x)+\epsilon}.
\end{equation}
Because $F_t$ is a soft attention distribution and $G_t$ is a binary or soft downsampled mask, these are soft-overlap scores rather than standard binary segmentation scores.

\paragraph{Intuition.}
CAMA answers the question: \emph{is the decoder looking at the object it is supposed to segment?} A higher CAMA value means more attention mass lies inside the target mask. A lower value means that the decoder may be using background, boundary, contextual, or irrelevant regions to produce the mask.

\paragraph{Interpretation.}
CAMA should be interpreted comparatively rather than absolutely. In SAM-style decoders, attention is often distributed over object boundaries, contextual landmarks, and support regions rather than being concentrated only inside the object. Therefore, even a good segmentation can have a modest CAMA value. The most meaningful use of CAMA in this paper is to compare methods under the same dataset, decoder, resolution, and prompt protocol. A method with consistently higher CAMA is better aligned with the target object, but low CAMA alone does not prove segmentation failure.

\paragraph{Direction.}
Higher is better. We report CAMA-Dice and CAMA-IoU with $\uparrow$ arrows in the result tables.

\paragraph{Ground-truth dependence.}
With ground truth, CAMA measures attention--target alignment. Without ground truth, replacing $G_t$ with the predicted mask measures attention--prediction agreement. The latter is useful as a self-consistency diagnostic but cannot detect the case where both the prediction and attention have drifted together to the wrong object.

\subsection{Attention Anchor Drift}
\label{app:aad-explanation}

\paragraph{Definition.}
Attention Anchor Drift (AAD) measures how far the current attention map has moved away from the initial anchor attention map:
\begin{equation}
  \mathrm{AAD}(F_t,F_0)=\mathrm{JS}(F_t\|F_0).
\end{equation}
Here $F_0$ is the attention map obtained from the first decoder call for the object, after alignment to the initial centroid frame. In practice, the initial attention map acts as an anchor describing where the decoder originally coupled to the object.

\paragraph{Intuition.}
AAD answers the question: \emph{how much has the decoder's attention moved away from its initial object-specific reference?} In a stable iterative segmentation loop, the prompt may adapt slightly from slice to slice, but the decoder should remain coupled to the same object. Large AAD indicates that the internal attention pattern has drifted away from the original object anchor.

\paragraph{Interpretation.}
A low AAD indicates stable attention anchoring. A high AAD indicates that the decoder may have shifted attention toward distractors, neighboring structures, or background regions. AAD is especially important in closed-loop prompting because prompt errors can accumulate: once the predicted mask shifts, the next prompt shifts, which can further shift attention. Thus, AAD is a direct internal signature of decoder coupling drift.

\paragraph{Direction.}
Lower is better. We report AAD with a $\downarrow$ arrow.

\paragraph{Ground-truth dependence.}
AAD is fully ground-truth-free. It only requires cached attention maps from the decoder. This makes it one of the most useful deployment-time monitoring metrics in the paper.

\paragraph{Edge cases.}
AAD can be high for legitimate reasons if the object undergoes a large appearance or shape change across slices. For this reason, AAD should be interpreted together with TCS and segmentation quality. A method is problematic when AAD increases while TCS decreases and prompt drift energy increases, because this combination suggests unstable feedback rather than natural object evolution.

\subsection{Temporal Coupling Stability}
\label{app:tcs-explanation}

\paragraph{Definition.}
Temporal Coupling Stability (TCS) measures attention consistency between consecutive decoder calls:
\begin{equation}
  \mathrm{TCS}(F_t,F_{t-1}) = 1 - \mathrm{JS}(F_t\|F_{t-1}).
\end{equation}
Because JS divergence is bounded, TCS is also bounded. Higher values indicate that the decoder's attention map changes smoothly over time.

\paragraph{Intuition.}
TCS answers the question: \emph{does the decoder remain coupled to a similar visual region from one iteration to the next?} In volumetric EM data, adjacent slices usually contain similar object structure. Therefore, abrupt attention changes are often a symptom of feedback instability.

\paragraph{Interpretation.}
High TCS indicates stable decoder coupling. Low TCS indicates that attention changes sharply between iterations. In the proposed framework, TCS is the clearest indicator of whether closed-loop prompting is dynamically stable. This is why \method{} closes the largest fraction of the oracle gap for TCS: proximal anchoring directly suppresses abrupt prompt-induced attention changes.

\paragraph{Direction.}
Higher is better. We report TCS with a $\uparrow$ arrow.

\paragraph{Ground-truth dependence.}
TCS is fully ground-truth-free. It only requires the current and previous attention maps. This makes it suitable for online monitoring during inference.

\paragraph{Edge cases.}
Very high TCS is not always desirable. If TCS is nearly one because attention is completely frozen, the method may fail to adapt to genuine object motion or shape change. Therefore, TCS should be interpreted as a stability metric, not as a complete measure of segmentation accuracy. The desired behavior is high-but-not-rigid TCS: stable enough to prevent drift, but flexible enough to track real object changes.

\subsection{Prompt Drift Energy}
\label{app:pde-explanation}

\paragraph{Definition.}
Prompt Drift Energy (PDE) measures the geometric drift of the prompt relative to the initial prompt:
\begin{equation}
  \mathrm{PDE}(p_t,p_0)
  = \frac{\|c_t-c_0\|_2^2}{\mathrm{diag}(I)^2}
  + \lambda_b\left(1-\mathrm{IoU}(b_t,b_0)\right),
\end{equation}
where $c_t$ is the current centroid, $b_t$ is the current bounding box, $\mathrm{diag}(I)$ is the image diagonal used for scale normalization, and $\lambda_b$ controls the relative contribution of bounding-box drift.

\paragraph{Intuition.}
PDE answers the question: \emph{how much has the prompt moved away from its initial object geometry?} Since the prompt is the input to the next decoder call, prompt drift is the external geometric counterpart of internal attention drift.

\paragraph{Interpretation.}
Low PDE means that the prompt trajectory remains close to the initial object anchor. High PDE means that the centroid or bounding box has moved substantially. In closed-loop prompting, high PDE is dangerous because each prompt is generated from the previous mask. A small mask error can shift the next prompt; the shifted prompt can worsen the next mask; and this cycle can amplify over iterations.

\paragraph{Direction.}
Lower is better. We report PDE with a $\downarrow$ arrow.

\paragraph{Ground-truth dependence.}
PDE is ground-truth-free. It only requires the sequence of prompts. Therefore, it can be monitored in deployment without annotations.

\paragraph{Edge cases.}
If the true object moves substantially across slices, PDE may increase even when segmentation is correct. In the MitoEM setting, adjacent slices are structurally similar enough that large prompt drift is usually undesirable. For datasets with stronger object motion, PDE should be interpreted relative to an expected motion model or normalized by estimated object displacement.

\subsection{Distractibility Leakage Ratio}
\label{app:dlr-explanation}

\paragraph{Definition.}
Distractibility Leakage Ratio (DLR) measures how much attention leaks outside the target object compared with attention inside the target object:
\begin{equation}
  \mathrm{DLR}(F_t,G_t)
  = \min\left(
  \frac{\sum_x F_t(x)(1-G_t(x))}{\sum_x F_t(x)G_t(x)+\epsilon},
  \tau_{\mathrm{clamp}}
  \right),
\end{equation}
where $\tau_{\mathrm{clamp}}=10^4$. We report the logarithmic value,
\begin{equation}
  \log\mathrm{DLR}=\log\left(\mathrm{DLR}+\epsilon\right),
\end{equation}
for numerical stability and readability.

\paragraph{Intuition.}
DLR answers the question: \emph{is the decoder spending more attention outside the object than inside it?} If most attention lies inside the target, DLR is low. If attention is dominated by background or distractor structures, DLR is high.

\paragraph{Interpretation.}
DLR is useful for identifying distractibility. In EM images, many structures have similar texture and contrast, so a decoder may attend to nearby membranes, mitochondria, or background patterns. A high DLR indicates that the decoder's attention is not object-specific. However, DLR can be large when the object is very small, because the denominator $\sum_x F_t(x)G_t(x)$ can be tiny. This motivates both clamping and log scaling.

\paragraph{Direction.}
Lower is better. We report $\log$DLR with a $\downarrow$ arrow.

\paragraph{Ground-truth dependence.}
With ground truth, DLR measures attention leakage outside the true object. Without ground truth, using the predicted mask measures leakage outside the model's own prediction. The latter can detect inconsistency between attention and output, but it cannot detect drift if both attention and prediction have moved to the same wrong region.

\paragraph{Edge cases.}
DLR is sensitive to object size and mask downsampling. For very small objects, a small spatial misalignment between attention and mask can cause a large DLR. Therefore, DLR should usually be read together with CAMA and PICS rather than used alone.

\subsection{Causal Coupling Drop}
\label{app:ccd-explanation}

\paragraph{Definition.}
Causal Coupling Drop (CCD) estimates whether the attended image tokens are functionally important for the output mask. We first compute the baseline segmentation quality $\mathrm{IoU}_{\mathrm{base}}$. Then we identify the top-$k$ attended image tokens according to the head-averaged final token-to-image attention map, replace their image-token features with the per-token mean feature, and run the decoder again. Let the resulting segmentation quality be $\mathrm{IoU}_{\mathrm{ablated}}$. CCD is
\begin{equation}
  \mathrm{CCD}=\max\left(0,\mathrm{IoU}_{\mathrm{base}}-\mathrm{IoU}_{\mathrm{ablated}}\right).
\end{equation}
In our experiments, $k=5\%$ of image tokens.

\paragraph{Intuition.}
CCD answers the question: \emph{does the prediction actually depend on the regions that the decoder attends to?} CAMA and DLR are correlational: they compare attention with a mask. CCD is counterfactual: it asks what happens if the most-attended tokens are removed or neutralized.

\paragraph{Interpretation.}
High CCD means that ablating the most-attended tokens substantially damages the prediction. This indicates that the decoder is functionally coupled to those tokens. Low CCD means that the prediction changes little after ablation, suggesting that the attention map may not be causally important or that the decoder is relying on distributed redundant evidence.

\paragraph{Direction.}
Higher is generally better when the attention is also well aligned with the target. We report CCD with a $\uparrow$ arrow.

\paragraph{Ground-truth dependence.}
CCD uses an IoU comparison and therefore depends on a reference mask. With ground truth, it measures causal importance for true segmentation quality. Without ground truth, it can be computed against the original prediction as a consistency measure, but this deployment-time version should be interpreted more cautiously.

\paragraph{Why mean replacement rather than zeroing?}
Zeroing tokens can create an out-of-distribution feature pattern that the decoder never sees during normal inference. Mean replacement is a softer intervention: it removes token-specific information while preserving the approximate feature scale. This makes CCD a controlled counterfactual rather than a destructive corruption.

\paragraph{Edge cases.}
High CCD is not automatically good. If attention has drifted to a wrong object and the decoder relies on that wrong region, CCD can still be high. Therefore, CCD should be interpreted together with CAMA, DLR, and segmentation quality. The most desirable regime is high CCD plus high CAMA and low DLR.

\subsection{Prompt-Image Coupling Score}
\label{app:pics-explanation}

\paragraph{Definition.}
Prompt-Image Coupling Score (PICS) is a composite metric designed to summarize multiple aspects of decoder coupling in a single scalar:
\begin{equation}
  \mathrm{PICS}
  = \mathrm{CAMA}_{\mathrm{Dice}}
  \cdot \mathrm{TCS}
  \cdot (1-\mathrm{AAD})
  \cdot \left(1-\frac{\mathrm{DLR}}{\tau_{\mathrm{clamp}}}\right).
\end{equation}
All terms are oriented so that larger values indicate better coupling.

\paragraph{Intuition.}
PICS answers the question: \emph{is the decoder simultaneously aligned, stable, anchored, and non-distracted?} It is intentionally multiplicative: if any one component is poor, the final score decreases. This reflects the idea that reliable closed-loop segmentation requires several conditions at once. Attention should overlap the target, remain stable over time, avoid drifting away from the anchor, and avoid leaking heavily to distractors.

\paragraph{Interpretation.}
A high PICS indicates strong overall prompt-image coupling. A low PICS identifies a failure in at least one coupling dimension. Because PICS is multiplicative, it is conservative: it does not allow excellent performance on one component to fully compensate for failure on another. This makes it useful as a summary diagnostic for drift.

\paragraph{Direction.}
Higher is better. We report PICS with a $\uparrow$ arrow.

\paragraph{Ground-truth dependence.}
PICS inherits the ground-truth dependence of CAMA and DLR. With ground truth, it measures overall coupling to the true object. Without ground truth, it becomes a self-consistency score between attention and prediction, combined with ground-truth-free stability terms.

\paragraph{Edge cases.}
Because PICS includes CAMA, it can remain low even when TCS and AAD are excellent if attention is distributed outside the object interior. This behavior is intentional: PICS is designed to reward stable and target-specific coupling, not merely temporal smoothness.

\subsection{Attention Entropy}
\label{app:ae-explanation}

\paragraph{Definition.}
Attention entropy (AE) measures the spatial uncertainty of the decoder attention map:
\begin{equation}
  \mathrm{AE}(F_t)=-\sum_x F_t(x)\log(F_t(x)+\epsilon).
\end{equation}
For comparability across maps, entropy may also be normalized by $\log(H_fW_f)$.

\paragraph{Intuition.}
AE answers the question: \emph{is the decoder attention concentrated or diffuse?} Low entropy means that attention is concentrated on a small number of spatial tokens. High entropy means that attention is spread across many tokens.

\paragraph{Interpretation.}
Entropy has no universally correct direction. Concentrated attention can indicate confident object-specific coupling, but it can also indicate over-fixation on a wrong distractor. Diffuse attention can indicate uncertainty, but it can also reflect the legitimate use of broad contextual cues. In this paper, AE is used primarily as a descriptive diagnostic and as a correlation variable for segmentation quality.

\paragraph{Ground-truth dependence.}
AE is ground-truth-free because it only requires the attention map.

\subsection{Self--Cross Attention Correlation}
\label{app:sca-explanation}

\paragraph{Definition.}
Spearman self--cross correlation (SCA) measures rank correlation between decoder self-attention structure and token-to-image cross-attention structure. Let $S_t$ be the relevant mask-token row or aggregated mask-token interaction from the final decoder self-attention, and let $C_t$ be the corresponding flattened token-to-image cross-attention summary. After resizing or aggregating to comparable vectors, SCA is
\begin{equation}
  \mathrm{SCA}=\rho_{\mathrm{Spearman}}(S_t,C_t),
\end{equation}
where $\rho_{\mathrm{Spearman}}$ is the Spearman rank correlation.

\paragraph{Intuition.}
SCA answers the question: \emph{are the decoder's internal token interactions consistent with where it attends in the image?} A high correlation suggests that the decoder's token-level reasoning and image-level attention are coordinated. A low or negative correlation suggests that self-attention and cross-attention are decoupled.

\paragraph{Interpretation.}
SCA is useful as an internal consistency metric. It does not directly measure segmentation quality, but it can reveal whether different parts of the decoder are moving coherently during iterative prompting. In a stable feedback loop, one expects prompt-token interactions and image-token attention to remain coordinated.

\paragraph{Ground-truth dependence.}
SCA is ground-truth-free. It only requires cached decoder attention tensors.

\subsection{Summary of Metric Roles}
\label{app:metric-summary-table}

The decoder-coupling metrics play complementary diagnostic roles. CAMA-Dice and CAMA-IoU measure soft overlap between decoder attention and the object mask; without ground truth they become attention--prediction self-consistency checks. AAD measures long-term divergence from the initial anchor attention and is fully ground-truth-free. TCS measures similarity between consecutive attention maps and is also ground-truth-free. PDE measures centroid and box drift in prompt space and can be computed from prompts alone. $\log$DLR measures attention leakage outside the object; without ground truth it should be interpreted only as leakage outside the predicted mask. CCD measures the counterfactual importance of top-attended tokens and is strongest when evaluated against ground truth, though a prediction-consistency variant is possible. PICS combines alignment, stability, anchoring, and leakage into a conservative overall coupling score. AE describes whether attention is concentrated or diffuse, while SCA measures agreement between decoder self-attention and cross-attention structure; both AE and SCA are ground-truth-free internal diagnostics.

\subsection{How the Metrics Should Be Read Together}
\label{app:metric-joint-interpretation}

No single metric fully characterizes closed-loop decoder behavior. The metrics are designed to be read as a diagnostic panel:
\begin{itemize}[leftmargin=*,itemsep=2pt,topsep=2pt]
  \item \textbf{AAD + PDE} diagnose drift. If both increase, the prompt trajectory and decoder attention are moving away from the initial object anchor.
  \item \textbf{TCS} diagnoses temporal smoothness. If TCS decreases while AAD increases, the loop is not merely adapting; it is becoming dynamically unstable.
  \item \textbf{CAMA + DLR} diagnose spatial target specificity. High CAMA and low DLR indicate that attention is concentrated on the object rather than distractors.
  \item \textbf{CCD} tests whether attention is functionally important. It separates attention maps that are merely visually plausible from attention maps that the decoder actually uses.
  \item \textbf{PICS} provides a conservative overall coupling score. It is most useful for summarizing whether a method is simultaneously stable, anchored, aligned, and non-distracted.
\end{itemize}

The characteristic failure pattern of standard iterative prompting is high AAD, low TCS, high PDE, and reduced PICS. This pattern means that the decoder's internal attention is drifting, the prompt geometry is drifting, and the feedback loop is amplifying small errors over iterations. The characteristic behavior of \method{} is lower AAD, higher TCS, lower PDE, and higher PICS, indicating that proximal anchoring stabilizes the dynamic part of the loop while preserving enough flexibility for slice-to-slice adaptation.

\subsection{Why Decoder-Coupling Metrics Are Needed in Addition to IoU and Dice}
\label{app:why-decoder-metrics}

Standard segmentation metrics such as IoU, Dice, and HD95 evaluate only the final output mask. They do not explain \emph{why} an iterative foundation-model pipeline succeeds or fails. In closed-loop prompting, two methods can have similar IoU at an early iteration but very different internal stability. One method may remain coupled to the target and therefore continue to perform well; another may already show attention drift that only becomes visible in the output after several more iterations.

The decoder-coupling metrics address this gap. AAD and TCS expose internal attention dynamics before output failure becomes severe. PDE connects internal attention drift to external prompt drift. DLR and CAMA distinguish target-specific attention from distractor-driven attention. CCD tests whether attended regions causally affect the prediction. PICS combines these signals into a single conservative summary. Together, these metrics make it possible to analyze the iterative segmentation process as a feedback system rather than as a sequence of independent mask predictions.

\subsection{Aggregation and Statistical Reporting}
\label{app:metric-aggregation}

Metrics are computed per decoder call. A decoder call is indexed by method, slice, object, and iteration. For method comparisons, we use paired observations whenever possible: the same slice, object, and iteration are compared across methods. This pairing reduces variance due to object size, morphology, and slice-specific image quality.

For each metric, the main tables report the mean and standard deviation across paired decoder calls. Significance is assessed using two-sided paired permutation tests. The null hypothesis is that swapping method labels within each pair does not change the mean difference. This test is appropriate because it does not require normality of metric distributions, which is important for skewed quantities such as DLR and CCD.

When metrics are correlated with segmentation quality, we report both Pearson correlation and Spearman rank correlation. Pearson correlation measures linear association, while Spearman correlation measures monotonic association and is more robust to nonlinear scaling and outliers. The per-slice correlations show whether aggregate decoder behavior predicts aggregate segmentation quality. The per-call correlations test whether the same relationship holds at the level of individual decoder calls.

\subsection{Practical Deployment Guidance}
\label{app:metric-deployment-guidance}

At deployment time, ground-truth masks are unavailable. Therefore, the most practical online monitoring metrics are AAD, TCS, PDE, AE, and SCA. These can be computed directly from prompts and cached attention tensors. A practical monitoring pipeline can use the following rules:
\begin{itemize}[leftmargin=*,itemsep=2pt,topsep=2pt]
  \item If \textbf{AAD increases steadily}, the decoder is drifting away from its initial object anchor.
  \item If \textbf{TCS drops sharply}, the current prompt update may have caused an unstable attention jump.
  \item If \textbf{PDE increases sharply}, the prompt geometry may have moved too far from the object trajectory.
  \item If \textbf{AAD and PDE increase together}, the system should strengthen anchor regularization or trigger re-initialization.
  \item If \textbf{TCS is extremely high but segmentation stops adapting}, the system may be over-regularized and should reduce smoothness or anchor weight.
\end{itemize}

CAMA, DLR, CCD, and PICS can still be computed using the predicted mask in place of ground truth. In that setting, they should be treated as self-consistency checks rather than accuracy measures. For example, low predicted-mask CAMA indicates that the decoder attention and its own output disagree; however, high predicted-mask CAMA does not guarantee correctness because both attention and prediction may be wrong in the same way.


\end{document}